\documentclass[12pt]{article}
\usepackage{epsfig,amsmath,amssymb,amsfonts,amstext,amsthm,mathrsfs}
\usepackage{latexsym,graphics,epsf,epsfig,psfrag}
\usepackage{cite}
\usepackage{epstopdf}
\newcommand{\mmd}{\text{MMD}}

\newtheorem{definition}{\textbf{Definition}}
\newtheorem{corollary}{\textbf{Corollary}}
\newtheorem{lemma}{\textbf{Lemma}}
\newtheorem{theorem}{\textbf{Theorem}}
\newtheorem{proposition}{\textbf{Proposition}}
\newtheorem{remark}{\textbf{Remark}}

\newcommand{\nn}{\nonumber}

\newcommand{\mE}{\mathbb{E}}

\newcommand{\cH}{\mathcal{H}}

\newcommand{\cN}{\mathcal{N}}

\newcommand{\cP}{\mathcal{P}}
\newcommand{\cI}{\mathcal{I}}

\newcommand{\tq}{\tilde{q}}

\DeclareMathAlphabet{\matheuf}{U}{euf}{m}{n}

\begin{document}
\title{Nonparametric Detection of Anomalous Data Streams}
\author{Shaofeng Zou, Yingbin Liang, {\em Senior Member, IEEE},\\H. Vincent Poor, {\em Fellow, IEEE}, and Xinghua Shi
\let\thefootnote\relax\footnote{The material in this paper was presented
in part in ~\cite{zou2014unsupervised} at the 52th Annual Allerton Conference on Communication, Control, and Computing, Monticello, IL, Oct. 2014.}
\let\thefootnote\relax\footnote{The work of S. Zou and Y. Liang was supported by a National Science Foundation
CAREER Award under Grant CCF-10-26565. The work of H. V. Poor was supported by the National Science Foundation under Grants CNS-14-56793 and ECCS-13-43210. The work of X. Shi was partly supported by National Science Foundation under Grant IIS-1502172.}
\let\thefootnote\relax\footnote{Shaofeng Zou is with the Department of Electrical and Computer Engineering and Coordinated Science Laboratory,
University of Illinois at Urbana Champaign, Urbana, IL 61801 USA (email: szou3@illinois.edu).
Yingbin Liang are with the Department of Electrical
Engineering and Computer Science, Syracuse University, Syracuse, NY 13244 USA (email: yliang06@syr.edu).
H. Vincent Poor is with the Department of Electrical Engineering, Princeton University, Princeton, NJ 08544 USA (email: poor@princeton.edu).
Xinghua Shi is with the Department of Bioinformatics and Genomics, University of North Carolina at Charlotte, Charlotte, NC 28223 (email: xshi3@uncc.edu).}
}
\date{}
\maketitle

\begin{abstract}
A nonparametric anomalous hypothesis testing problem is investigated, in which there are totally $n$ sequences with $s$ anomalous sequences to be detected. Each typical sequence contains $m$ independent and identically distributed (i.i.d.) samples drawn from a distribution $p$, whereas each anomalous sequence contains $m$ i.i.d.\ samples drawn from a distribution $q$ that is distinct from $p$. The distributions $p$ and $q$ are assumed to be unknown in advance. Distribution-free tests are constructed using maximum mean discrepancy  as the metric, which is based on mean embeddings of distributions into a reproducing kernel Hilbert space. The probability of error is bounded as a function of the sample size $m$, the number $s$ of anomalous sequences and the number $n$  of sequences. It is then shown that with $s$ known, the constructed test is exponentially consistent if $m$ is greater than a constant factor of $\log n$, for any $p$ and $q$, whereas with $s$ unknown, $m$ should has an order strictly greater than $\log n$.
Furthermore, it is shown that no test can be consistent for arbitrary $p$ and $q$ if $m$ is less than a constant factor of $\log n$, thus the order-level optimality of the proposed test is established. Numerical results are provided to demonstrate that our tests outperform (or perform as well as) the tests based on other competitive approaches under various cases.
\end{abstract}
{\bf Key words:}
  Anomalous hypothesis testing, consistency, distribution-free tests, maximum mean discrepancy (MMD).

%
%

\section{Introduction}\label{sec:introduction}
\setcounter{page}{1}

In this paper, we study an anomalous hypothesis testing problem (see Figure~\ref{fig:model}), in which there are totally $n$ sequences out of which $s$ anomalous sequences need to be detected. Each \emph{typical} sequence consists of $m$ independent and identically distributed (i.i.d.) samples drawn from a distribution $p$, whereas each \emph{anomalous} sequence contains i.i.d.\ samples drawn from a distribution $q$ that is distinct from $p$. The distributions $p$ and $q$ are assumed to be unknown. The goal is to build distribution-free tests to detect the $s$ anomalous data sequences generated by $q$ out of all data sequences.

Solutions to such a problem are very useful in many applications. For example, in cognitive wireless networks, signals follow different distributions either $p$ or $q$ depending on whether the channel is busy or vacant. A major issue in such a network is to identify vacant channels out of a large number of busy channels based on their corresponding signals in order to utilize vacant channels for improving spectral efficiency. This problem was studied in \cite{Lai2011} and \cite{Tajer2013} under the assumption that $p$ and $q$ are known, whereas in this paper, we study the problem with unknown $p$ and $q$.  Other applications include detecting anomalous DNA sequences out of typical sequences, detecting virus infected computers from other virus free computers, and detecting slightly modified images from other untouched images.

The parametric model of the problem has been well studied, e.g., \cite{Lai2011}, in which it is assumed that the distributions  $p$ and $q$ are known in advance and can be exploited for detection. However, the nonparametric model is less explored, in which it is assumed that the distributions $p$ and $q$ are unknown and can be arbitrary. Recently, Li, Nitinawarat and Veeravalli proposed the divergence-based generalized likelihood tests in \cite{Li2013}, and characterized the error decay exponents of these tests. However, \cite{Li2013} studied only the case when the distributions $p$ and $q$ are discrete with finite alphabets, and their tests utilize empirical probability mass functions of $p$ and $q$.

In this paper, we study the nonparametric model, in which distributions $p$ and $q$ can be continuous and arbitrary. The major challenges to solve this problem (compared to the discrete case studied in \cite{Li2013}) lie in: (1) it is difficult to accurately estimate continuous distributions with limited samples for further anomalous hypothesis testing; (2) it is difficult to design low complexity tests with continuous distributions; and (3) building distribution-free consistent tests (and further guaranteeing exponential error decay) is challenging for arbitrary distributions.


Our approach adopts the {\em maximum mean discrepancy (MMD)} introduced in \cite{Gretton2012} as the distance metric between two distributions. The idea is to map probability distributions into a reproducing kernel Hilbert space (RKHS) (as proposed in \cite{Berl2004,Srip2010}) such that the distance between the two probabilities can be measured by the distance between their corresponding embeddings in the RKHS. MMD can be easily estimated based on samples, and hence yields low complexity tests.
In this paper, we apply MMD as a metric to construct our tests for detecting anomalous data sequences.
In contrast to consistency analysis in classical theory as in \cite{Li2013},
which assumes that the problem dimension (i.e., the number $n$ of sequences and the number $s$ of
anomalous sequences) is fixed and the sample size $m$ increases, our focus is on the regime
in which the problem dimension (i.e., $n$ and $s$) increases.
This is motivated by  applications, in which anomalous sequences are required to be detected out of a large number of typical data sequences. It is clear that as $n$   (and possibly   $s$) becomes large, it is increasingly challenging to consistently detect all anomalous sequences. It then requires that the sample size $m$  correspondingly increases in order to guarantee more accurate detection.  Hence, we are interested in characterizing how the sample size $m$ should scale with $n$ and $s$ in order to guarantee the consistency of our tests.

In this paper, we adopt the following notations to express asymptotic scaling of quantities with $n$:
\begin{list}{$\bullet$}{\topsep=0.ex \leftmargin=0.3in \rightmargin=0.in \itemsep =-0.05in}
\item $f(n)=O(g(n))$: there exist $k,n_0>0$ s.t.\ for all $ n>n_0$, $|f(n)|\leq k|g(n)|$;
\item $f(n)=\Omega(g(n))$: there exist $k,n_0>0$ s.t.\ for all $ n>n_0$, $f(n)\geq kg(n)$;
\item $f(n)=\Theta(g(n))$: there exist $k_1,k_2,n_0>0$ s.t.\ for all $ n>n_0$, $k_1g(n)\leq f(n)\leq k_2g(n)$;
\item $f(n)=o(g(n))$: for all $k>0$, there exists $n_0>0$ s.t.\ for all $ n>n_0$, $|f(n)|\leq kg(n)$;
\item $f(n)=\omega(g(n))$: for all $k>0$, there exists $n_0>0$ s.t.\ for all $ n>n_0$, $|f(n)|\geq k|g(n)|$.
\end{list}

\subsection{Main Contributions}
We summarize our main contributions as follows.

(1) We construct MMD-based distribution-free tests, which enjoy low computational complexity and are proven to be powerful for nonparametric detection.

(2) We analyze the performance guarantee for the proposed MMD-based test. We bound the probability of error as a function of the sample size $m$, the number $s$ of anomalous sequences, and the total number $n$ of sequences. We then show that with $s$ known, the constructed test is exponentially consistent if $m$ scales at the order $\Omega(\log n)$ for any $p$ and $q$, whereas with $s$ unknown, $m$ should scale at the order $\omega(\log n)$ (i.e., strictly larger than $\Omega(\log n)$). Thus, the lack of the information about $s$ results in an order-level increase in sample size $m$ needed for consistent detection. We further develop low complexity consistent tests by exploiting the asymptotic behavior of $s$ and $n$.

(3) We further derive a necessary condition which states that no test can be consistent for arbitrary $p$ and $q$ if $m$ scales at the order $O(\log n)$, thus establishing the order-level optimality of the MMD-based test.

(4) We provide an interesting example study, in which the distribution $q$ is the mixture of the distribution $p$ and the anomalous distribution $\tq$. In such a case, the anomalous sequence contains only sparse samples from the anomalous distribution. Our results for such a model quantitatively characterize the impact of the sparsity level of anomalous samples on the scaling behavior of the sample size $m$, in order to guarantee consistency of the proposed tests.

We provide numerical results to demonstrate our theoretical assertions and compare our tests with other competitive approaches. Our numerical results demonstrate that the MMD-based test has a better performance than the divergence-based generalized likelihood test proposed in \cite{Li2013} when the sample size $m$ is not very large. We also demonstrate that the MMD-based test outperforms (or performs as well as) other competitive tests including t-test, FR-Wolf test \cite{FRwolf1979}, FR-Smirnov test \cite{FRwolf1979}, Hall test \cite{Hall2002} as well as kernel density ratio (KDR) test \cite{Kanamori2012} and kernel Fisher discriminant analysis (KFDA) test \cite{Harchaoui2008}.


\subsection{Related Work}

In this subsection, we review relevant problems and explain their differences from our model. The parametric model of our problem with {\em known} $p$ and $q$ has been studied, e.g., in \cite{Lai2011}. The nonparametric model with unknown $p$ and $q$ were studied recently in \cite{Li2013}, where $p$ and $q$ are assumed to be {\em discrete} distributions. Our study addresses the general scenario in which $p$ and $q$ can be {\em arbitrary} (not necessarily discrete) and {\em unknown}. Furthermore, we allow the sample size to scale with the total number $n$ of sequences (which goes to infinity), whereas \cite{Li2013} studies the regime in which $n$ is fixed and only the sample size goes to infinity.

As generalization of the classical two-sample problem, which tests whether two sets of samples are generated from the same or different distributions, our problem involves much richer ingredients and more technical challenges. Our problem involves interplay of the number $n$ of sequences, the number $s$ of anomalous sequences, and the sample size $m$ to guarantee test consistency, whereas the two sample problem involves only the sample complexity. Furthermore, test consistency in our problem depends on the knowledge of the number of anomalous sequences, whereas the two sample problem does not have such an issue. These new issues naturally require considerably more technical efforts such as analysis of the MMD estimator via samples from mixed distributions, bounding the asymptotic behavior of difference between two MMD estimators, and development of necessary conditions on sample complexity.

A popular type of outlier detection problems have been widely studied in data mining, e.g., \cite{Patcha2007,Chan2009}, in which a number of data samples are given and outliers that are far away from other samples (typically in Euclidean distance) need to be detected. These studies typically do not assume underlying statistical models for data samples, whereas our problem assumes that data are drawn from either distribution $p$ or $q$. Thus, our problem is to detect an outlier {\em distribution} rather than an outlier data sample.

Another related but different model has been studied in \cite{Hero2006,Hero1999asymptotic,Zhao2009}, which tests whether a new sample is generated from the same distribution as a given set of training samples. Such a problem is binary composite hypothesis testing, whereas our problem is multi-hypothesis testing, detecting anomalous sequences out of a set of sequences that contain both typical and anomalous sequences. Furthermore, such a problem assumes availability of a training set of (typical) samples, whereas our problem does not assume any sample known to be typical in advance.

\subsection{Organization of the Paper}

The rest of the paper is organized as follows. In Section \ref{sec:model}, we describe the problem formulation. In Section \ref{sec:withoutreference}, we present our tests and theoretical results on the performance guarantee of these tests. In Section \ref{sec:converse}, we further present the necessary conditions to guarantee test consistency. In Section \ref{sec:simulations}, we provide numerical results. Finally in Section \ref{sec:conclusion}, we conclude our paper with remarks on future work.

\section{Problem Statement}\label{sec:model}


\begin{figure}[htb]
\begin{center}
  \includegraphics[width=6cm]{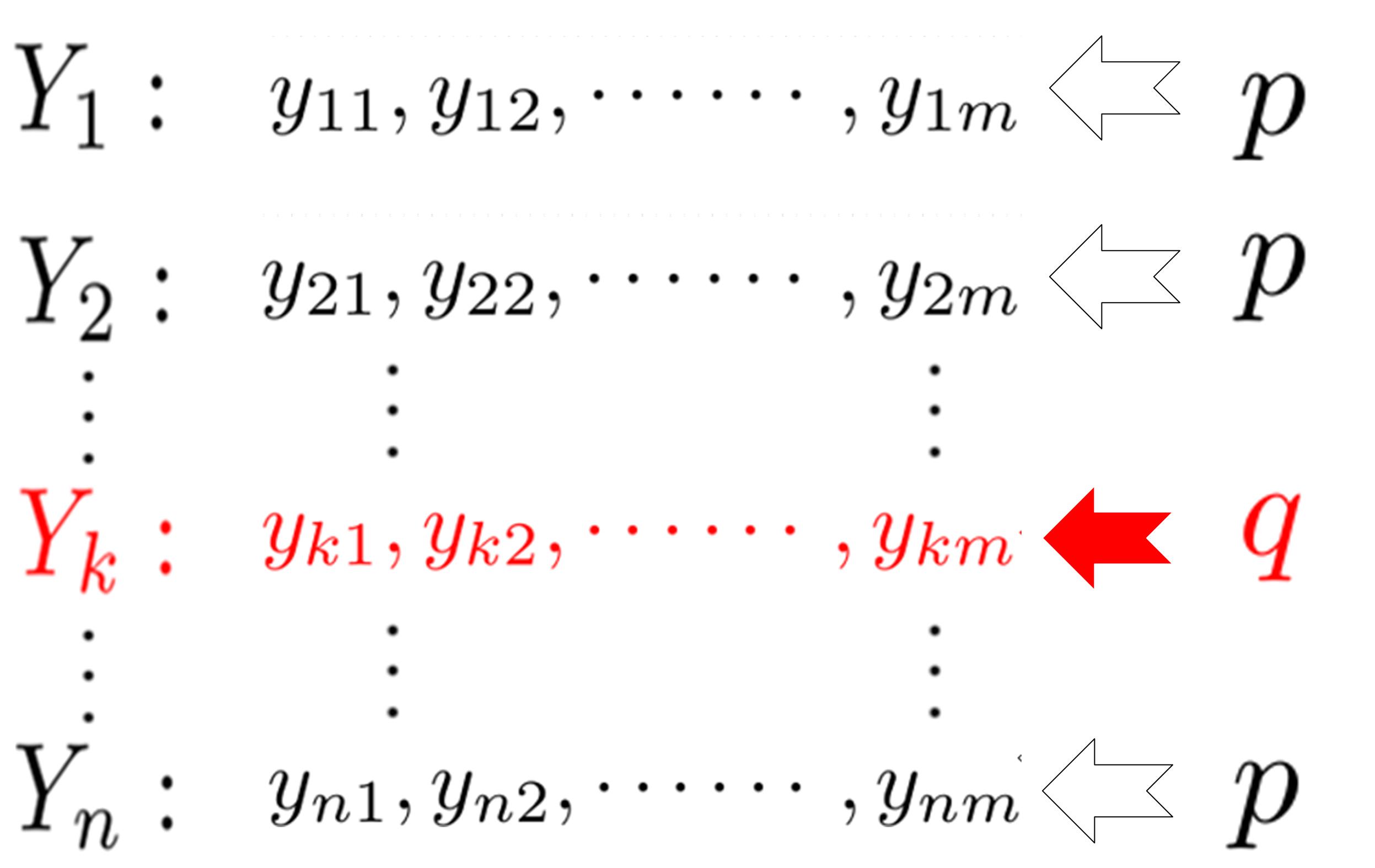}
  \caption{An anomalous hypothesis testing model with data sequences generated by typical distribution $p$ and anomalous distribution $q$.}
  \label{fig:model}
\end{center}
\end{figure}


We study an anomalous hypothesis testing problem (see Figure~\ref{fig:model}), in which there are in total $n$ data sequences denoted by $Y_k$ for $1\leq k\leq n$. Each data sequence $Y_k$ consists of $m$ i.i.d.\ samples $y_{k1},\ldots,y_{km}$ drawn from either a typical distribution $p$ or an anomalous distribution $q$, where $p\neq q$. In the sequel, we use the notation $Y_k:=(y_{k1},\ldots,y_{km})$. We assume that the distributions $p$ and $q$ are arbitrary and unknown in advance. Our goal is to build distribution-free tests to detect data sequences generated by the anomalous distribution $q$.

We assume that $s$ out of $n$ data sequences are anomalous, i.e., are generated by the anomalous distribution $q$. We study both cases with   $s$ known and unknown, respectively. We are interested in the asymptotical regime, in which the number $n$ of data sequences goes to infinity. We assume that the number $s$ of anomalous sequences satisfies $\frac{s}{n}\rightarrow \alpha$ as $n\rightarrow \infty$, where $0\leq\alpha\leq1$. This includes the following three cases: (1) $s$ is fixed, and nonzero as $n \rightarrow \infty$; (2) $s \rightarrow \infty$, but $\frac{s}{n}\rightarrow 0$ as $n \rightarrow \infty$; and (3) $\frac{s}{n}$ approaches to a positive constant, which is less than or equal to $1$. Some of our results are also applicable to the case with $s=0$, i.e., the null hypothesis in which there is no anomalous sequence. We will comment on such a case when the corresponding results are presented.



We next define the probability of detection error as the performance measure of tests. We let $\cI$ denote the set that contains indices of all anomalous data sequences. Hence, the cardinality $|\cI|=s$. We let $\hat{\cI}^n$ denote a sequence of index sets that contain indices of all anomalous data sequences claimed by a corresponding sequence of tests.
\begin{definition}\label{def:consistency}
A sequence of tests are said to be consistent if
\begin{equation}
    \lim_{n\rightarrow\infty}P_e=\lim_{n\rightarrow\infty}P\{\hat\cI^n\neq\cI^n\}=0.
\end{equation}
\end{definition}

We note that the above definition of consistency is with respect to the number $n$ of sequences instead of the number $m$ of samples. However, as $n$ becomes large (and possibly as $s$ becomes large), it is increasingly challenging to consistently detect all anomalous data sequences. It then requires that the number $m$ of samples becomes large enough in order to more accurately detect anomalous sequences. Therefore, the limit in the above definition in fact refers to the asymptotic regime, in which $m$ scales fast enough as $n$ goes to infinity in order to guarantee asymptotically small probability of error.


Furthermore, for a consistent test, it is also desirable that the error probability decays exponentially fast with respect to the number $m$ of samples.
\begin{definition}
A sequence of tests are said to be exponentially consistent if
\begin{equation}
    \liminf_{m\rightarrow\infty}-\frac{1}{m}\log P_e=\liminf_{m\rightarrow\infty}-\frac{1}{m}\log P\{\hat\cI^n\neq\cI^n\} > 0.
\end{equation}
\end{definition}



In this paper, our goal is to construct distribution-free tests for detecting anomalous sequences, and characterize the scaling behavior of $m$ with $n$ (and possibly $s$) so that the developed tests are consistent (and possibly exponentially consistent).

\textbf{An example with sparse anomalous samples.} In this paper, we also study an interesting example, in which the distribution $q$ is a mixture of the distribution $p$ with probability $1-\epsilon$ and an anomalous distribution $\tq$ with probability $\epsilon$, where $0 < \epsilon \leq 1$, i.e., $q=(1-\epsilon)p+\epsilon \tq$. It can be seen that if $\epsilon$ is small, the majority of samples in an anomalous sequence are drawn from the distribution $p$, and only sparse samples are drawn from the anomalous distribution $\tq$. The value of $\epsilon$ captures the sparsity level of anomalous samples. Here, $\epsilon$ can scale as $n$ increases, and is hence denoted by $\epsilon_n$. We study how $\epsilon_n$ affects the number of samples needed for consistent detection.

\section{Test and Performance Guarantee}\label{sec:withoutreference}

We adopt the {\em maximum mean discrepancy (MMD)} introduced in \cite{Gretton2012} as the distance metric to construct our test. More specifically, suppose each distribution $p$ belonging to $\cP$ (a set of probability distributions) is mapped to an element in the RKHS $\cH$ as follows
\[\mu_p(\cdot)=\mE_p [k(\cdot,x)]=\int k(\cdot,x)dp(x), \]
where $k(\cdot,\cdot)$ is the kernel function associated with $\cH$. It has been shown in \cite{Fuku2008,Srip2008} that the above mean embedding mapping is injective for many RKHSs such as those associated with Gaussian and Laplace kernels. The MMD between $p$ and $q$ is defined to be the distance between $\mu_p$ and $\mu_q$ in RKHS given by
\begin{equation}
\mmd[p,q]:=\|\mu_p-\mu_q\|_{\cH}.
\end{equation}
Due to the reproducing property of kernel, it can be easily shown that
\begin{flalign}\label{eq:mmdpq}
\mmd^2[p,q]=&\mE_{x,x'}[k(x,x')]-2\mE_{x,y}[k(x,y)] +\mE_{y,y'}[k(y,y')],
\end{flalign}
where $x$ and $x'$ have independent but the same distribution $p$, and $y$ and $y'$ have independent but the same distribution $q$. An unbiased estimator of $\mmd^2[p,q]$ based on $l_1$ samples of $X$ and $l_2$ samples of $Y$ is given as follows,
\begin{small}
\begin{flalign}\label{eq:mmdu}
&\mmd_u^2[X,Y]=\frac{1}{l_1(l_1-1)}\sum_{i=1}^{l_1}\sum_{j\neq i}^{l_1} k(x_i,x_j) +\frac{1}{l_2(l_2-1)}\sum_{i=1}^{l_2}\sum_{j\neq i}^{l_2} k(y_i,y_j)-\frac{2}{l_1l_2}\sum_{i=1}^{l_1}\sum_{j=1}^{l_2} k(x_i,y_j).
\end{flalign}
\end{small}

In this section, we design and analyze MMD-based tests for both cases with $s$ known and unknown, respectively. We then study the example with sparse anomalous samples.

\subsection{Known $s$}\label{sec:sufficient}

In this subsection, we consider the case with $s$ known. We start with a simple case with $s=1$, and then study the more general case, in which $\frac{s}{n}\rightarrow \alpha$ as $n\rightarrow \infty$, where $0\leq \alpha\leq1$.

Consider the case with $s=1$. For each sequence $Y_k$, we use $\overline Y_k$ to denote the $(n-1)m$ dimensional sequence that stacks all other sequences together, as given by
\[\overline Y_k=\{Y_1,\ldots,Y_{k-1},Y_{k+1},\ldots,Y_n\}.\]
We then compute $\mmd_u^2[Y_k,\overline Y_k]$ for $1\leq k\leq n$. It is clear that if $Y_k$ is the anomalous sequence, then $\overline Y_k$ is fully composed of typical sequences. Hence, $\mmd_u^2[Y_k,\overline Y_k]$ is a good estimator of $\mmd^2[p,q]$, which is a positive constant. On the other hand, if $Y_k$ is a typical sequence, $\overline Y_k$ is composed of $n-2$ sequences generated by $p$ and only one sequence generated by $q$. As $n$ increases, the impact of the anomalous sequence on $\overline Y_k$ is negligible, and $\mmd_u^2[Y_k,\overline Y_k]$ should be asymptotically close to zero. Based on the above understanding, we construct the following test when $s=1$. The sequence $k^*$ is claimed to be anomalous if
\begin{flalign}\label{eq:test_1_withoutref}
    k^*=\arg\max_{1\leq k\leq n}\mmd_u^2[Y_k,\overline Y_k].
\end{flalign}

The following proposition characterizes the condition under which the above test is consistent.
\begin{proposition}\label{thm:s1withoutref}
Consider the anomalous hypothesis testing model with one anomalous sequence, i.e., $s=1$. Suppose the test \eqref{eq:test_1_withoutref} applies a bounded kernel with $0\leq k(x,y)\leq K$ for any $(x,y)$.
Then, the probability of error is upper bounded as follows,
\begin{flalign}\label{pe1}
  P_e\leq \exp\Big(\log n-\frac{m(\mmd^2[p,q]-\xi)^2}{16K^2(1+ \Theta(\frac{1}{n}))}\Big),
\end{flalign}
where $\xi  $ is a constant which can be picked arbitrarily close to zero.
Furthermore, the test \eqref{eq:test_1_withoutref} is exponentially consistent if
\begin{flalign}\label{eq:suff}
m\geq\frac{16K^2(1+\eta)}{\mmd^4[p,q]}\log n,
\end{flalign}
where $\eta$ is any positive constant. 
\end{proposition}
\begin{proof}
See Appendix \ref{proof:s1_withoutref}.
\end{proof}

Proposition \ref{thm:s1withoutref} implies that for the scenario with one anomalous sequence, $\Omega(\log n)$ samples are sufficient to guarantee  consistent detection.



We next consider the case with $s\geq 1$. More specifically, we consider the case with $\frac{s}{n}\rightarrow \alpha$ as $n\rightarrow \infty$, where $0\leq \alpha< \frac{1}{2}$.  Although we focus on the case with $\alpha<\frac{1}{2}$, the case with $\alpha > \frac{1}{2}$ is similar, with the roles of $p$ and $q$ being exchanged. We first study the case with $s$  known. Our test is a natural generalization of the test \eqref{eq:test_1_withoutref} except now the test picks the sequences with the largest $s$ values of $\mmd_u^2[Y_k,\overline Y_k]$, which is given by
\begin{flalign}\label{eq:test_s_withoutref}
\hat{\cI}=&\{k:\mmd_u^2[Y_k,\overline Y_k] \text{ is among the $s$ largest}  \text{ values of $\mmd_u^2[Y_i,\overline Y_i]$ for } i=1,\ldots,n\}.
\end{flalign}

The following theorem characterizes the condition under which the above test is consistent.
\begin{theorem}\label{thm:swithoutref}
Consider the anomalous hypothesis testing model with $s$ anomalous sequences, where $\frac{s}{n}\rightarrow \alpha$ as $n\rightarrow \infty$ and $0\leq \alpha< \frac{1}{2}$. Assume the value of $s$ is known. Further assume that the test \eqref{eq:test_s_withoutref} applies a bounded kernel with $0\leq k(x,y)\leq K$ for any $(x,y)$.
Then the probability of error is upper bounded as follows,
\begin{flalign}\label{pe2}
  P_e\leq \exp\Big(\log((n-s)s)-\frac{m((1-2\alpha)\mmd^2[p,q]-\xi)^2}{16K^2(1+ \Theta(\frac{1}{n}))}\Big),
\end{flalign}
where $\xi$ is a constant which can be picked arbitrarily close to zero.
Furthermore, the test \eqref{eq:test_s_withoutref} is exponentially consistent for any $p$ and $q$ if
\begin{flalign}\label{eq:mnoref}
m\geq\frac{16K^2(1+\eta)}{(1-2\alpha)^2\mmd^4[p,q]}\log(s(n-s)),
\end{flalign}
where $\eta$ is any positive constant. 
\end{theorem}
\begin{proof}
See Appendix \ref{proof:s_withoutref}.
\end{proof}

We note that $\log((n-s)s)=\Theta (\log n)$, for $1\leq s< n$. Hence, Theorem \ref{thm:swithoutref} implies that even with $s$  anomalous sequence, the test \eqref{eq:test_s_withoutref} requires only $\Omega(\log n)$ samples in each data sequence in order to guarantee consistency of the test. Hence, the increase of  $s$   does not affect the order-level requirement on the sample size $m$.
We further note that Theorem \ref{thm:swithoutref} is also applicable to the case in which $\alpha > \frac{1}{2}$ simply with the roles of $p$ and $q$ exchanged.

\begin{remark}
For the case with $\frac{s}{n}\rightarrow 0$, as $n\rightarrow \infty$, we can also build a test with reduced computational complexity as follows. For each $Y_k$, instead of using $n-1$ sequences to build $\overline Y_k$ as in the test \eqref{eq:test_s_withoutref}, we take any $l$ sequences out of the remaining $n-1$ sequences to build a sequence $\widetilde Y_k$, such that $\frac{l}{n}\rightarrow 0$ and $\frac{s}{l}\rightarrow 0$ as $n\rightarrow \infty$. Such an $l$ exists for any $s$ and $n$ satisfying $\frac{s}{n}\rightarrow 0$ (e.g., $l=\sqrt{sn}$). It can be shown that using $\widetilde Y_k$ to replace $\overline Y_k$ in the test \eqref{eq:test_s_withoutref} still leads to consistent detection under the same condition given in Theorem \ref{thm:swithoutref}. Since $l$ is much smaller than $n$, computational complexity is substantially reduced.
\end{remark}

%
We note that Theorem \ref{thm:swithoutref} (which includes Proposition \ref{thm:s1withoutref} as a special case) characterizes the conditions to guarantee test consistency for a pair of fixed but unknown distributions $p$ and $q$. Hence, the condition \eqref{eq:mnoref} depends on the underlying distributions $p$ and $q$. In fact, such a condition further yields the following condition that guarantees the test to be universally consistent for arbitrary $p$ and $q$.
\begin{proposition}[Universal Consistency]\label{prop:achievesknown}
Consider the anomalous hypothesis testing problem, where $\frac{s}{n}\rightarrow \alpha$ as $n\rightarrow \infty$ and $0\leq \alpha< \frac{1}{2}$. Assume $s$ is known. Further assume that the test \eqref{eq:test_s_withoutref} applies a bounded kernel with $0\leq k(x,y)\leq K$ for any $(x,y)$. Then the test \eqref{eq:test_s_withoutref} is universally consistent for any arbitrary pair of $p$ and $q$, if
\begin{flalign}
  m=\omega(\log n).
\end{flalign}
\end{proposition}
\begin{proof}
This result follows from \eqref{eq:mnoref} and the facts that $\log((n-s)s)=\Theta (\log n)$ and $\mmd[p,q]$ is constant for any given $p$ and $q$.
\end{proof}


\subsection{Unknown $s$}
In this subsection, we consider the case, in which the value of $s$ is unknown. And we focus on the scenario that   $\frac{s}{n}\rightarrow 0$, as $n\rightarrow \infty$. This includes two cases: (1) $s$ is fixed and (2) $s\rightarrow \infty$ and $\frac{s}{n} \rightarrow 0$ as $n \rightarrow \infty$.
Without knowledge of $s$, the test in \eqref{eq:test_s_withoutref} is not applicable anymore, because it depends on the value of $s$.

In order to build a test now, we first observe that for each $k$, although $\overline Y_k$ contains mixed samples from $p$ and $q$, it is dominated by samples from $p$ due to the above assumption on $s$. Thus, for large enough $m$ and $n$, $\mmd_u^2[Y_k,\overline Y_k]$ should be close to zero if $Y_k$ is drawn from $p$, and should be far away enough from zero (in fact, close to $\mmd^2[p,q]$) if $Y_k$ is drawn from  $q$. Based on this understanding, we construct the following test:
\begin{flalign}\label{test:sunknown_nonref}
\widehat{\mathcal{I}}=\{k:\mmd_u^2[Y_k,\overline Y_k]>\delta_n\}
\end{flalign}
where $\delta_n \rightarrow 0$ and $\frac{s^2}{n^2 \delta_n} \rightarrow 0$ as $n \rightarrow \infty$. The reason for the condition $\frac{s^2}{n^2 \delta_n} \rightarrow 0$ is to guarantee that $\delta_n$ converges to $0$ more slowly than $\mmd_u^2[Y_k,\overline Y_k]$ with $Y_k$ drawn from $p$ so that as $n$ goes to infinity, $\delta_n$ asymptotically falls between $\mmd_u^2[Y_k,\overline Y_k]$ with $Y_k$ drawn from $p$ and $\mmd_u^2[Y_k,\overline Y_k]$ with $Y_k$ drawn from $q$.
We note that the scaling behavior of $s$ as $n$ increases needs to be known in order to pick $\delta_n$ for the test. This is reasonable to assume because mostly in practice the scale of anomalous data sequences can be estimated based on domain knowledge.

The following theorem characterizes the condition under which the test \eqref{test:sunknown_nonref} is consistent.
\begin{theorem}\label{thm:sunknownnonref}
Consider the anomalous hypothesis testing model with $s$ anomalous sequences, where $\frac{s}{n}\rightarrow 0$, as $n\rightarrow \infty$. Assume that $s$ is unknown in advance. Further assume that the test \eqref{test:sunknown_nonref} adopts a threshold $\delta_n$ such that $\delta_n\rightarrow 0$ and $\frac{s^2}{n^2\delta_n} \rightarrow 0$, as $n\rightarrow \infty$, and the test applies a bounded kernel with $0\leq k(x,y)\leq K$ for any $(x,y)$.
Then the probability of error is upper bounded as follows:
\begin{flalign}
  P_e\leq  &\exp\bigg( \log s  -\frac{m(\mmd^2[p,q]  -  \delta_n)^2}{16K^2(1+\Theta(\frac{1}{n}))}  \bigg) +\exp\bigg(\log(n-s)   -\frac{m(\delta_n-\mE\big[\mmd_u^2[Y_k,\overline Y_k]\big])^2}{16K^2(1+ \Theta(\frac{1}{n}))}  \bigg).
\end{flalign}
Furthermore, the test \eqref{test:sunknown_nonref} is consistent if
\begin{flalign}\label{eq:sunknown}
m\geq 16(1+\eta)K^2& \max \Big\{ \frac{\log (\max\{1,s\})}{(\mmd^2[p,q] -  \delta_n)^2} \;,  \frac{\log(n-s)}{(\delta_n-\mE\big[\mmd_u^2[Y,\overline Y]\big])^2} \Big\},
\end{flalign}
where $\eta$ is any positive constant. In the above equation, $\mE[\mmd_u^2[Y,\overline Y]]$ is a constant, where $Y$ is a sequence generated by $p$ and $\overline Y$ is a stack of $(n-1)$ sequences with $s$ sequences generated by $q$ and the remaining sequences generated by $p$.
\end{theorem}
\begin{proof}
See Appendix \ref{proof:sunknownnonref}.
\end{proof}

We note that Theorem \ref{thm:sunknownnonref} is also applicable to the case with $s=0$, i.e., the null hypothesis when there is no anomalous sequence. We further note that the test \eqref{test:sunknown_nonref} is not exponentially consistent.
In fact, when there is no null hypothesis (i.e., $s>1$), an exponentially consistent test can be built as follows. For each subsect $\mathcal S$ of ${1,\ldots,n}$, we compute $\mmd_u^2[Y_{\mathcal S},\overline Y_{\mathcal S}]$, and the test finds the set of indices corresponding to the largest average value. However, for such a test to be consistent, $m$ needs to scale linearly with $n$, which is not desirable.

Theorem \ref{thm:sunknownnonref} implies that $m$ should be in the order $\omega(\log n)$ to guarantee test consistency, because $\frac{s}{n}\rightarrow 0$ and $\delta_n \rightarrow 0$ as $n \rightarrow \infty$. Compared to the case with $s$ known (for which it is sufficient for $m$ to scale at the order $\Theta(\log n)$), the threshold on $m$ has order-level increase due to lack of the knowledge of $s$. Furthermore, the above understanding on the order-level condition on $m$ also yields the following sufficient condition for the test to be universally consistent.
\begin{proposition}[Universal Consistency]\label{prop:achievesunknown}
Consider the anomalous hypothesis testing problem, where $\frac{s}{n}\rightarrow 0$, as $n\rightarrow \infty$. We assume that $s$ is unknown in advance. Further assume that the test \eqref{test:sunknown_nonref} adopts a threshold $\delta_n$ such that $\delta_n\rightarrow 0$ and $\frac{s^2}{n^2\delta_n} \rightarrow 0$, as $n\rightarrow \infty$, and the test applies a bounded kernel with $0\leq k(x,y)\leq K$ for any $(x,y)$. Then the test \eqref{test:sunknown_nonref} is universally consistent for any arbitrary pair of $p$ and $q$, if
\begin{flalign}
  m=\omega(\log n).
\end{flalign}
\end{proposition}
Comparison between Proposition \ref{prop:achievesunknown} with Proposition \ref{prop:achievesknown} implies that the knowledge of $s$ does not affect the order-level sample complexity to guarantee a test to be universally consistent.



\subsection{Example with Sparse Anomalous Samples}

We study the example with  the anomalous distribution $q=(1-\epsilon_n)p+\epsilon_n \tq$ as we introduce in Section \ref{sec:model}. The following result characterizes the impact of sparsity level $\epsilon_n$ on the scaling behavior of $m$ to guarantee consistent detection.
\begin{corollary}\label{coro:mix_withoutref}
Consider the model with the typical distribution $p$ and the anomalous distribution $q=(1-\epsilon_n)p+\epsilon_n \tq$, where $0 < \epsilon_n \leq 1$. If $s$ is known, then the test \eqref{eq:test_s_withoutref} is consistent if
\begin{flalign}
m\geq\frac{16K^2(1+\eta)}{(1-2\alpha)^2\epsilon_n^4\mmd^4[p,\tq]}\log(s(n-s)),
\end{flalign}
where $\eta$ is any positive constant.

If $s$ is unknown, then
the test  \eqref{test:sunknown_nonref}  is consistent if
\begin{flalign}
m\geq 16(1+\eta)K^2& \max \Big\{ \frac{\log (\max\{1,s\})}{(\epsilon_n^2\mmd^2[p,\tq] -  \delta_n)^2} \;,  \frac{\log(n-s)}{(\delta_n-\mE\big[\mmd_u^2[Y,\overline Y]\big])^2} \Big\},
\end{flalign}
where $\eta$ is any positive constant, $\frac{s^2\epsilon_n^2}{n^2\delta_n}\rightarrow 0$ and $\frac{\delta_n}{\epsilon_n^2}\rightarrow 0$ as $n\rightarrow \infty$, $Y$ is a sequence generated by $p$, and $\overline Y$ is a stack of $(n-1)$ sequences with $s$ sequences generated by $\tq$ and the remaining sequences generated by $p$.
\end{corollary}
\begin{proof}
The proof follows from Theorems \ref{thm:swithoutref} and \ref{thm:sunknownnonref} by substituting:
\begin{flalign}\label{eq:coro:mix_ref}
\mmd^2[p, q]&=\mathbb{E}_{x,x'}[k(x,x')]-2\mathbb{E}_{x,y}[k(x,y)]+\mathbb{E}_{y,y'}[k( y,y')]\nn\\
&=\mE_{x,x'}[k(x,x')]-2(1-\epsilon_n)\mE_{x,x'}[k(x,x')] -2\epsilon_n\mE_{x,\tilde y}[k(x,\tilde y)]\nn\\
&\quad+(1-\epsilon_n)^2\mE_{x,x'}[k(x,x')] +2\epsilon_n(1-\epsilon_n)\mE_{x,\tilde y}[k(x,\tilde y)]+\epsilon_n^2\mE_{\tilde y,\tilde y'}[k(\tilde y,\tilde y')]\nn\\
&=\epsilon_n^2\mmd^2[p,\tilde q],
\end{flalign}
where $x$ and $x'$ are independent but have the same distribution $p$, $y$ and $y'$ are independent but have the same distribution $q$, and $\tilde y$ and $\tilde y'$ are independent but have the same distribution $\tilde q$.
\end{proof}

Corollary \ref{coro:mix_withoutref} implies that if $\epsilon_n$ is a constant, then the scaling behavior of $m$ needed for consistent detection does not change. However, if $\epsilon_n\rightarrow 0$ as $n \rightarrow \infty$, i.e., anomalous sequences contain more sparse anomalous samples, then $m$ needs to scale faster with $n$ in order to guarantee consistent detection. This is reasonable because the sample size $m$ should have a higher order to cancel out the impact of the increasingly sparse anomalous samples in each anomalous sequence. Corollary \ref{coro:mix_withoutref} explicitly captures such tradeoff between the sample size $m$ and the sparsity level $\epsilon_n$ of anomalous samples in addition to $n$ and $s$.

\section{Necessary Condition and Optimality}\label{sec:converse}

In Section \ref{sec:withoutreference}, we characterize sufficient conditions on the sample size $m$ under which the MMD-based test is guaranteed to be consistent for any distribution pair $p$ and $q$. In this section, we  characterize conditions under which no test is universally consistent for arbitrary $p$ and $q$. We first study the case with $s=1$ for which we develop our key idea of the proof. We then generalize our study to the case with $s \geq 1$.
\begin{proposition}\label{thm:conv}
  Consider the anomalous hypothesis testing problem with one anomalous sequence. If the sample size $m$  satisfies
  \begin{flalign}\label{eq:necessary}
    m=O(\log n),
  \end{flalign}
  then there exists no test that is universally consistent for any arbitrary distribution pair $p$ and $q$.
\end{proposition}
\begin{proof}
See Appendix \ref{sec:proofcon}. The idea of the proof is to show that for a certain distribution pair $p$ and $q$, even the optimal parametric test (with known $p$ and $q$) is not consistent under the condition given in the theorem. This thus implies that under the same condition, no nonparametric test is universally consistent for arbitrary $p$ and $q$.
\end{proof}

We now generalize our result to the case with $s \ge 1$, and provide the following proposition.
\begin{proposition}\label{thm:convs}
  Consider the anomalous hypothesis testing problem with $s$ anomalous sequences. If the sample size $m$ satisfies
  \begin{flalign}\label{eq:convs}
    m= O\left(\frac{\log\frac{n}{s}}{s}\right),
  \end{flalign}
then there exists no test that is universally consistent for arbitrary  distribution pair $p$ and $q$.
\end{proposition}
\begin{proof}
It can be shown that the probability of error of this problem is lower bounded by a special scenario, in which anomalous sequences can only be a group of $s$ sequences with consecutive indices, i.e., one of the following possibilities: the $(is+1)$-th to $(i+1)s$-th sequences, for $i=0,\ldots, \lfloor \frac{n}{s}\rfloor-1$. Hence, there are $\lfloor \frac{n}{s}\rfloor$ candidates. Such a specific scenario can be viewed as the problem of detecting one anomalous sequence with length $ms$ out of $\lfloor \frac{n}{s}\rfloor$ sequences. The proposition then follows from arguments similar to those used to prove Proposition \ref{thm:conv}.
\end{proof}

The sufficient and necessary conditions on sample complexity that we derive so far establish  the following performance optimality for the MMD-based test.
\begin{theorem}[Optimality]
Consider the nonparametric anomalous hypothesis testing problem with  $s \geq 1$. For $s$ being known and unknown, the MMD-based test \eqref{eq:test_s_withoutref} (under the conditions in Propositions \ref{prop:achievesknown}) and the test \eqref{test:sunknown_nonref}  (under the conditions in Proposition \ref{prop:achievesunknown}) are respectively order-level optimal in sample complexity required to guarantee universal  consistency for arbitrary $p$ and $q$.
\end{theorem}
\begin{proof}
The proof follows by comparing Propositions \ref{prop:achievesknown} and \ref{prop:achievesunknown} with Proposition \ref{thm:convs} and observing the fact that $m= O(\log n)$ in Proposition \ref{thm:convs} for finite $s$.
\end{proof}

\section{Numerical Results}\label{sec:simulations}

In this section, we provide numerical results to demonstrate our theoretical assertions, and compare our MMD-based tests with a number of other tests. We also apply our test to a real data set.

We first demonstrate our theorem on sample complexity. We note that although the following experiment is performed for chosen distributions $p$ and $q$, our tests are nonparametric and do not exploit the information about $p$ and $q$. We choose the distribution $p$ to be Gaussian with mean zero and variance one, i.e., $\cN(0,1)$, and choose the anomalous distribution $q$ to be Laplace distribution with mean one and variance one. We use the Gaussian kernel $k(x,x')=\exp(-\frac{|x-x'|^2}{2\sigma^2})$ with $\sigma=1$.
We set $s=1$. We run the test for cases with   $n=40$ and $100$, respectively. In Figure~\ref{fig:non_ref_1}, we plot how the probability of error changes with $m$. For illustrational convenience, we normalize $m$ by $\log n$, i.e., the horizontal axis represents $\frac{m}{\log n}$. It is clear from the figure that when $\frac{m}{\log n}$ is above a certain threshold, the probability of error converges to zero, which is consistent with our theoretical results. Furthermore, for different values of $n$, the two curves drop to zero almost at the same threshold. This observation confirms Proposition \ref{thm:s1withoutref}, which states that the threshold on $\frac{m}{\log n}$ depends only on the bound $K$ of the kernel and MMD of the two distributions. Both quantities are constant for the two values of $n$.

\begin{figure}[htbp]
\centering
\includegraphics[width=4in]{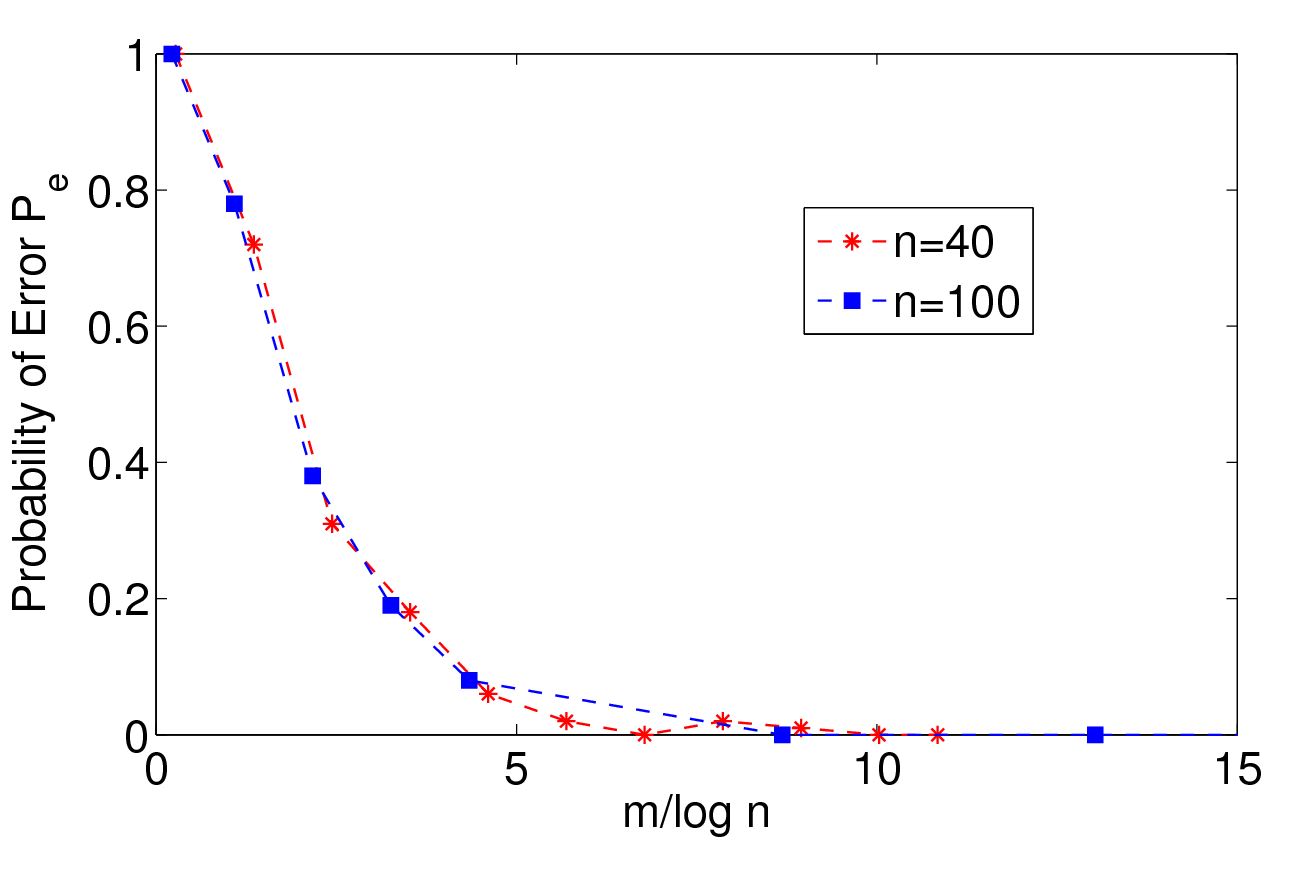}
\caption{The performance of the MMD-based test.}
\label{fig:non_ref_1}
\end{figure}

We next compare the MMD-based test with the divergence-based generalized likelihood test developed in \cite{Li2013}. Since the test in \cite{Li2013} is applicable only when the distributions $p$ and $q$ are discrete and have finite alphabets, we set the distributions $p$ and $q$ to be binary with $p$ having probability 0.3 to take ``0" (and probability 0.7 to take ``1"), and $q$ having probability 0.7 to take ``0" (and probability 0.3 to take ``1"). We let $s=1$ and assume that $s$ is known. We let $n=50$. In Figure~\ref{fig:binary_mmd}, we plot the probability of error as a function of the sample size $m$. It can be seen that the MMD-based test outperforms the divergence-based generalized likelihood test. We note that it has been shown in \cite{Li2013} that the generalized likelihood test has optimal convergence rate in the limiting case when $n$ is infinite. Our numerical comparison, on the other hand, demonstrates that the MMD-based test performs as well as or even better than the generalized likelihood test for moderate $n$.

\begin{figure}[htbp]
\centering
\includegraphics[width=4in]{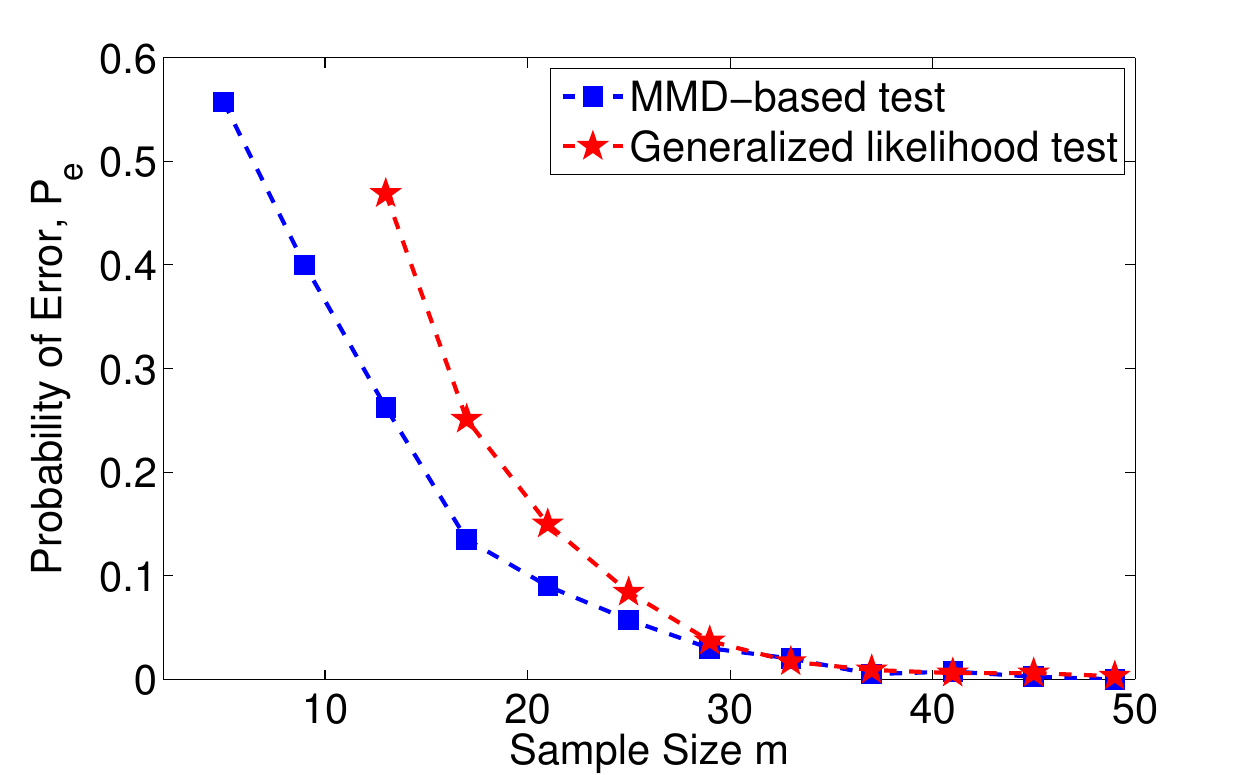}
\caption{Comparison of the MMD-based test with divergence-based generalized likelihood test.}
\label{fig:binary_mmd}
\end{figure}

We finally compare the performance of the MMD-based test with a few other competitive tests on a real data set. We choose the collection of daily maximum temperature of Syracuse (New York, USA) in July from 1993 to 2012 as the typical data sequences, and the collection of daily maximum temperature of Makapulapai (Hawaii, USA) in May from 1993 to 2012 as anomalous sequences. Here, each data sequence contains daily maximum temperatures of a certain day across twenty years from 1993 to 2012. In our experiment, the data set contains 32 sequences in total, including one temperature sequence of Hawaii and 31 sequences of Syracuse. The probability of error is averaged over all cases with each using one sequence of Hawaii as the anomalous sequence. Although it seems easy to detect the sequence of Hawaii out of the sequences of Syracuse, the temperatures we compare for the two places are in May for Hawaii and July for Syracuse, during which the two places have approximately the same mean in temperature. In this way, it may not be easy to detect the anomalous sequence (in fact, some tests do not perform well as shown in Figure~\ref{fig:realdata}).

\begin{figure}[htb]
\centering
\includegraphics[width=4in]{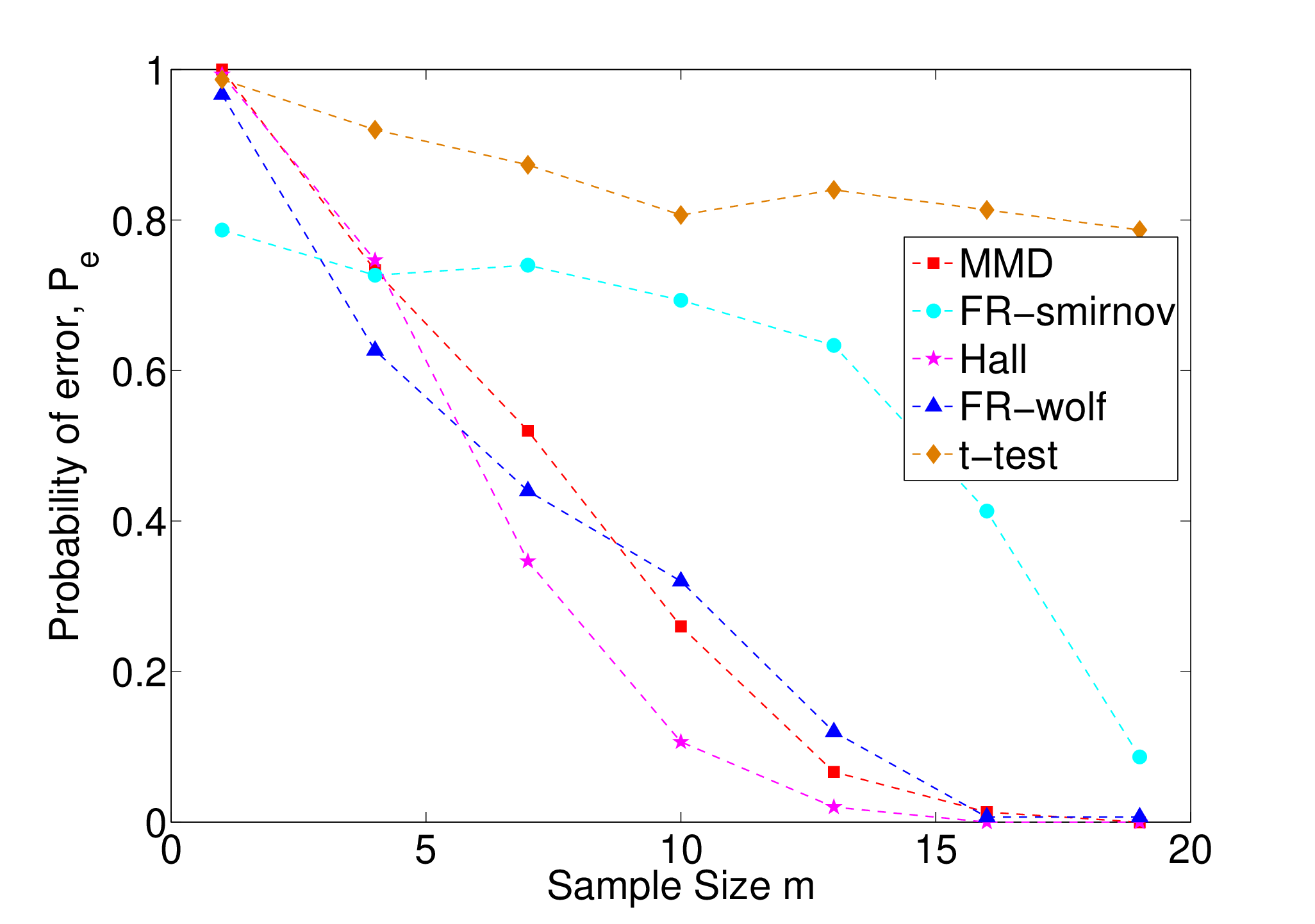}
  \caption{Comparison of the MMD-based test with four other tests on a real data set.}
  \label{fig:realdata}
\end{figure}%


We first compare the performance of the MMD-based test with t-test, FR-Wolf test, FR-Smirnov test, and Hall test on the above data set. For the MMD-based test, we use the Gaussian kernel with $\sigma=1$. In Figure~\ref{fig:realdata}, we plot the probability of error as a function of the length of sequence $m$ for all tests. It can be seen that the MMD-based test, Hall test, and FR-wolf test have the best performances, and all of the three tests are consistent with the probability of error converging to zero as $m$ goes to infinity. Furthermore, comparing to Hall and FR-wolf tests, the MMD-based test has the lowest computational complexity.

\begin{figure}[!htbp]
\centering
\includegraphics[width=4in]{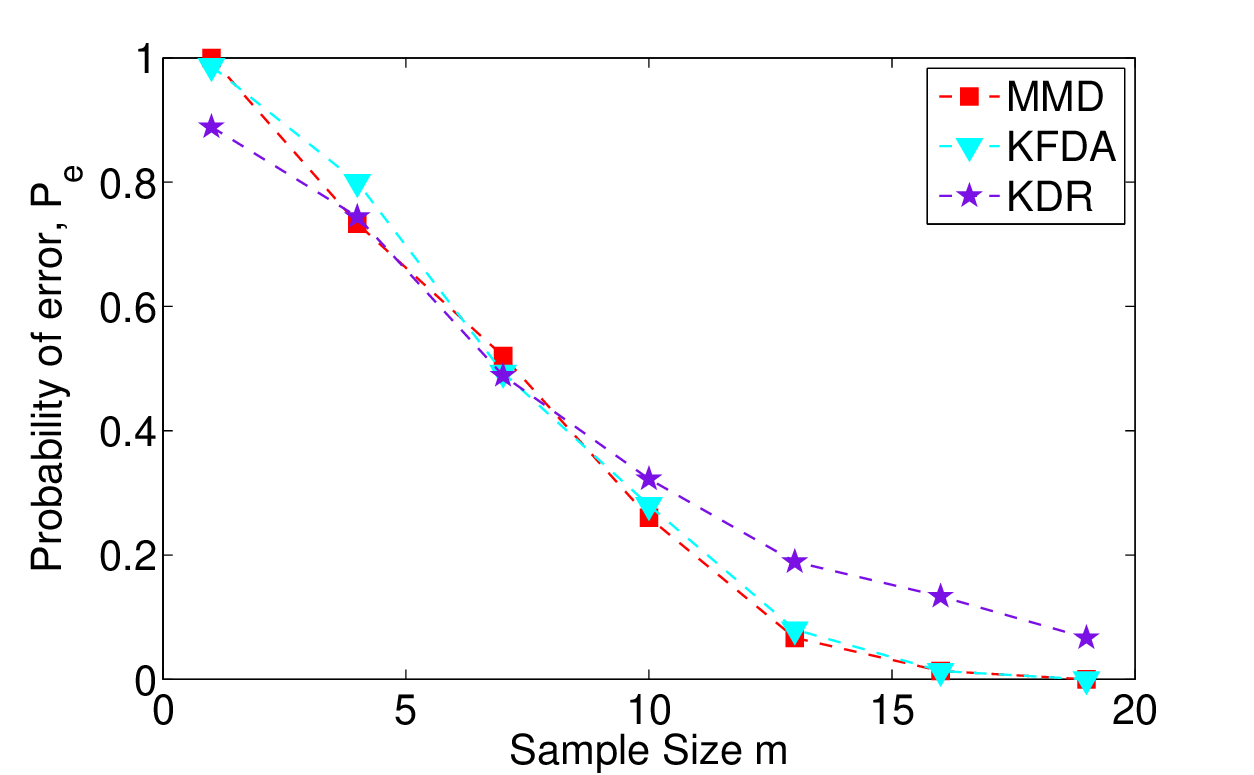}
  \caption{Comparison of the MMD-based test with two other kernel-based tests on a real data set.}
  \label{fig:realdata_k}
\end{figure}

We further compare the performance of MMD-based test with the kernel-based tests KFDA and KDR for the same data set. For all three tests, we use Gaussian kernel with $\sigma=1$. In Figure~\ref{fig:realdata_k}, we plot the probability of error as a function of the length of sequence for all tests. It can be seen that all tests are consistent with the probability of error converging to zero as $m$ increases, and the MMD-based test has the best performance among the three tests.

%

\section{Conclusion}\label{sec:conclusion}

In this paper, we have investigated a nonparametric anomalous hypothesis testing problem, in which typical and anomalous data sequences contain i.i.d.\ samples drawn from different distributions $p$ and $q$, respectively. We have built MMD-based distribution-free tests to detect anomalous sequences. We have characterized the scaling behavior of the sample size $m$ as the total number $n$ of sequences goes to infinity in order to guarantee consistency of the developed tests. We have further characterized the conditions under which no test  is universally consistent for arbitrary $p$ and $q$, and thus established that our proposed tests are order-level optimal. Our study of this problem demonstrates a useful application of the mean embedding of distributions and MMD, and we believe that such an approach can be applied to solving various other nonparametric problems.


\vspace{0.5in}
\appendix

\noindent {\Large \textbf{Appendix}}



\section{Proof of Proposition \ref{thm:s1withoutref}}\label{proof:s1_withoutref}
We first introduce the McDiarmid's inequality which is useful in bounding the probability of error in our proof.
\begin{lemma}[McDiarmid's Inequality]\label{mcdiarmid}
Let $f:\mathcal{X}^m\rightarrow \mathbb{R}$ be a function such that for all $i\in\{1,\ldots,m\}$, there exist $c_i<\infty$ for which
\begin{equation}
\underset{X\in\mathcal{X}^m, \tilde{x}\in \mathcal X}{sup}|f(x_1,\ldots,x_m)-f(x_1,\ldots x_{i-1},\tilde x,x_{i+1},\ldots,x_m)|\leq c_i.
\end{equation}
Then for all probability measure $p$ and every $\epsilon>0$,
\begin{equation}
P_X\bigg(f(X)-E_X(f(X))>\epsilon\bigg)<\exp\left(-\frac{2\epsilon^2}{\sum_{i=1}^mc_i^2}\right),
\end{equation}
where $X$ denotes $(x_1,\ldots,x_m)$, $E_X$ denotes the expectation over the $m$ random variables $x_i\thicksim p$, and $P_X$ denotes the probability over these $m$ variables.
\end{lemma}

In order to analyze the probability of error for the test \eqref{eq:test_1_withoutref}, without loss of generality, we assume that the first sequence is the anomalous sequence generated by the anomalous distribution $q$. Hence,
\begin{flalign}
P_e&=P(k^*\neq 1)\nn\\
&=P\bigg(\exists k\neq 1: \mmd_u^2[Y_k,\overline Y_k]>\mmd_u^2[Y_1,\overline Y_1]\bigg)\nn\\
&\leq \sum_{k=2}^nP\bigg(\mmd_u^2[Y_k,\overline Y_k]>\mmd_u^2[Y_1,\overline Y_1]\bigg).
\end{flalign}
For notational convenience, we stack $Y_1,\ldots, Y_n$ into a $nm$ dimensional row vector $Y=\{y_i,1\leq i\leq nm\}$, where $Y_k=\{y_{(k-1)m+1},\ldots,y_{km}\}$. And we define  $n'=(n-1)m$.
We then have,
\begin{flalign}\label{mmd:y1y1bar}
&\mmd_u^2[Y_1,\overline Y_1]=\frac{1}{m(m-1)}\sum_{\substack{i,j=1\\i\neq j}}^{m,m}k(y_i,y_j)+\frac{1}{n'(n'-1)}\sum_{\substack{i,j=m+1\\i\neq j}}^{nm}k(y_i,y_j)-\frac{2}{mn'}\sum_{\substack{i=1\\j=m+1}}^{m,nm}k(y_i,y_j).
\end{flalign}

For $2\leq k\leq n$, we have,
\begin{small}
\begin{flalign}\label{mmd:ykykbar}
\mmd_u^2&[Y_k,\overline Y_k]=\frac{1}{m(m-1)}\sum_{\substack{i,j=(k-1)m+1\\ i\neq j}}^{km,km}k(y_i,y_j)+\frac{1}{n'(n'-1)}\bigg(\sum_{\substack{i,j=1\\i\neq j}}^{m,m}k(y_i,y_j)+2\sum_{\substack{i=1\\j=m+1}}^{m,(k-1)m}k(y_i,y_j)\nn\\
& + 2\sum_{\substack{i=1\\j=km+1}}^{m,nm}k(y_i,y_j) +\sum_{\substack{i,j=m+1\\i\neq j}}^{(k-1)m,(k-1)m}k(y_i,y_j) +\sum_{\substack{i,j=km+1\\i\neq j}}^{nm,nm}k(y_i,y_j) +2\sum_{\substack{i=m+1\\j={km+1}}}^{(k-1)m,nm}k(y_i,y_j)\bigg)\nn\\
  &-\frac{2}{mn'}\bigg(\sum_{\substack{i=1\\j=(k-1)m+1}}^{m,km}k(y_i,y_j)+\sum_{\substack{i=m+1\\j=(k-1)m+1}}^{(k-1)m,km}k(y_i,y_j)+\sum_{\substack{i=(k-1)m+1\\j=km+1}}^{km,nm}k(y_i,y_j)\bigg).
\end{flalign}
\end{small}

We define
\[\Delta_k=\mmd_u^2[Y_k,\overline Y_k]-\mmd_u^2[Y_1,\overline Y_1].\]
It can be shown that,
\[\mE [\mmd_u^2[Y_1,\overline Y_1]]=\mmd^2[p,q],\]
and
\begin{flalign}
\mE[\mmd_u^2[Y_k,\overline Y_k]]&=\mE_{x,x'}k(x,x')+\frac{1}{(n-1)m((n-1)m-1)}\bigg(m(m-1)\mE_{y,y'}k(y,y')\nn\\
&+2m^2(n-2)\mE_{x,y}k(x,y)+((n-2)m-1)(n-2)m\mE_{x,x'}k(x,x')\bigg)\nn\\
&-\frac{2}{(n-1)m^2}\bigg(m^2\mE_{x,y}k(x,y)+(n-2)m^2\mE_{x,x'}k(x,x')\bigg)\nn\\
&\rightarrow 0,\text{ as }n\rightarrow \infty,
\end{flalign}
where $x$ and $x'$ are independent but have the same distribution $p$, $y$ and $y'$ are independent but have the same distribution $q$.
Hence, there exists a constant $\xi$ that satisfies
\begin{flalign}\label{eq:ieq1}
\mE[\mmd_u^2[Y_k,\overline Y_k]]<\xi<\mmd^2[p,q],
\end{flalign}
for large enough $n$. Here, $\xi$ can be arbitrarily close to zero as $n\rightarrow \infty$.

We next divide the entries in $\{y_1,\ldots,y_{nm}\}$ into three groups: $Y_1=\{y_1,\ldots,y_m\}$, $Y_k=\{y_{(k-1)m+1}\ldots,y_{km}\}$, and $\widehat{Y_k}$ that contains the remaining entries. We define $Y_{-a}$ as $Y$ with the $a$-th component $y_a$ being removed.

For $1\leq a\leq m$, $y_a$ affects $\Delta_k$ through the following terms
\begin{flalign}
&\frac{1}{n'(n'-1)}\bigg(2\sum_{\substack{j=1\\j\neq a}}^{m}k(y_a,y_j)+2\sum_{j=m+1}^{(k-1)m}k(y_a,y_j)+2\sum_{j=km+1}^{nm}k(y_a,y_j)\bigg)\nn\\
&-\frac{2}{mn'}\sum_{j=(k-1)m+1}^{km}k(y_a,y_j)-\frac{2}{m(m-1)}\sum_{\substack{j=1\\k\neq a}}^mk(y_a,y_j)+\frac{2}{mn'}\sum_{j=m+1}^{nm}k(y_a,y_j).
\end{flalign}
Hence, for $1\leq a\leq m$, we have
\begin{flalign}\label{eq:thm3prfp1}
|\Delta_k\big(Y_{-a},y_a\big)-\Delta_k\big(Y_{-a},y'_a\big)|\leq \frac{4K}{m}\bigg(1+\Theta\bigg(\frac{1}{n}\bigg)\bigg).
\end{flalign}

For $(k-1)m+1\leq a\leq km$, $y_a$ affects $\Delta_k$ through
\begin{flalign}
&\frac{2}{m(m-1)}\sum_{\substack{j=(k-1)m+1\\j\neq a}}^{km}k(y_a,y_j)-\frac{2}{mn'}\bigg(\sum_{i=1}^mk(y_i,y_a)+\sum_{i=m+1}^{(k-1)m}k(y_i,y_a)+\sum_{j=km+1}^{nm}k(y_a,y_j)\bigg)\nn\\
&-\frac{2}{n'(n'-1)}\sum_{\substack{j=m+1\\j\neq a}}^{nm}k(y_a,y_j)+\frac{2}{mn'}\sum_{i=1}^mk(y_a,y_i).
\end{flalign}
Hence, for $(k-1)m+1\leq a\leq km$, we have
\begin{flalign}\label{eq:thm3prfp2}
|\Delta_k\big(Y_{-a},y_a\big)-\Delta_k\big(Y_{-a},y'_a\big)|\leq \frac{4K}{m}\bigg(1+ \Theta  \bigg(\frac{1}{n}\bigg)\bigg).
\end{flalign}

For $m+1\leq a \leq (k-1)m$ and $km+1\leq a\leq nm$, $y_a$ affects $\Delta_k$ through
\begin{flalign}
&\frac{2}{n'(n'-1)}\bigg(\sum_{i=1}^mk(y_i,y_a)+\sum_{\substack{i=m+1\\i\neq a}}^{(k-1)m}k(y_i,y_a)+\sum_{j=km+1}^{nm}k(y_a,y_j)\bigg)-\frac{2}{mn'}\sum_{j=(k-1)m+1}^{km}k(y_a,y_j)\nn\\
&-\frac{2}{n'(n'-1)}\sum_{\substack{j=m+1\\j\neq a}}^{nm}k(y_a,y_j)+\frac{2}{mn'}\sum_{i=(k-1)m+1}^{km}k(y_i,y_a).
\end{flalign}
Hence, for $m+1\leq a \leq (k-1)m$ or $km+1\leq a\leq nm$, we have
\begin{flalign}\label{eq:thm3prfp3}
|\Delta_k\big(Y_{-a},y_a\big)-\Delta_k\big(Y_{-a},y'_a\big)|\leq \frac{1}{m} \Theta\bigg(\frac{1}{n}\bigg).
\end{flalign}
We further derive the following probability,
\begin{small}
\begin{flalign}
&P\bigg(\mmd_u^2[Y_k,\overline Y_k]>\mmd_u^2[Y_1,\overline Y_1]\bigg)\nn\\
&=P\bigg(\mmd_u^2[Y_k,\overline Y_k]-\mmd_u^2[Y_1,\overline Y_1]+\mmd^2[p,q]>\mmd^2[p,q]\bigg)\nn\\
&\overset{(a)}{\leq} P\bigg(\mmd_u^2[Y_k,\overline Y_k]-\mmd_u^2[Y_1,\overline Y_1]+\mmd^2[p,q]-\mE[\mmd_u^2[Y_k,\overline Y_k]]>\mmd^2[p,q]-\xi\bigg),
\end{flalign}
\end{small}
where (a) follows from \eqref{eq:ieq1}.

Combining \eqref{eq:thm3prfp1}, \eqref{eq:thm3prfp2}, \eqref{eq:thm3prfp3}, and applying McDiarmid's inequality, we have,
\begin{flalign}
P\bigg(&\mmd_u^2[Y_k,\overline Y_k]>\mmd_u^2[Y_1,\overline Y_1]\bigg)\nn\\
&\leq \exp\bigg(-\frac{2(\mmd^2[p,q]-\xi)^2}{2m\frac{16K^2}{m^2}(1+ \Theta(\frac{1}{n}))+\frac{1}{m} \Theta(\frac{1}{n})}\bigg)\nn\\
&=\exp\bigg(-\frac{m(\mmd^2[p,q]-\xi)^2}{16K^2(1+ \Theta(\frac{1}{n}))}\bigg)
\end{flalign}
Hence,
\begin{flalign}
P_e\leq \exp\bigg(\log n-\frac{m(\mmd^2[p,q]-\xi)^2}{16K^2(1+ \Theta(\frac{1}{n}))}\bigg).
\end{flalign}
Since $\xi$ can be picked arbitrarily close to zero, we conclude that if
\begin{flalign}
m\geq\frac{16K^2(1+\eta)}{\mmd^4[p,q]}\log n,
\end{flalign}
where $\eta$ is any positive constant, then $P_e\rightarrow 0$ as $n\rightarrow \infty$. It is also clear that if the above condition is satisfied, $P_e$ converges to zero exponentially fast with respect to $m$. This completes the proof.

\section{Proof of Theorem \ref{thm:swithoutref}}\label{proof:s_withoutref}
We analyze the performance of the test \eqref{eq:test_s_withoutref}. Without loss of generality, we assume that the first $s$ sequences are anomalous and are generated from distribution $q$. Hence, the probability of error can be bounded as,
\begin{flalign}
P_e=&P\bigg(\exists k>s:\mmd_u^2[Y_k,\overline Y_k] > \min_{1\leq l\leq s}\mmd_u^2[Y_l,\overline Y_l]\bigg)\nn\\
\leq &\sum_{k=s+1}^n\sum_{l=1}^sP\bigg(\mmd_u^2[Y_k,\overline Y_k]>\mmd_u^2[Y_l,\overline Y_l]\bigg).
\end{flalign}

Using the fact that $\frac{s}{n}\rightarrow \alpha$, where $0\leq \alpha<\frac{1}{2}$, and using \eqref{mmd:y1y1bar} and \eqref{mmd:ykykbar}, we can show that
\begin{flalign}
\mE\big[\mmd_u^2[Y_l,\overline Y_l]\big]\rightarrow (1-\alpha)^2\mmd^2[p,q],
\end{flalign}
as $n\rightarrow \infty$ for $1\leq l\leq s$, and
\begin{flalign}
\mE\big[\mmd_u^2[Y_k,\overline Y_k]\big]\rightarrow \alpha^2\mmd^2[p,q],
\end{flalign}
as $n\rightarrow \infty$ for $s+1\leq k\leq n$. Hence, there exists a constant $\xi$ such that
\[0<\xi<(1-\alpha)^2\mmd^2[p,q]-\alpha^2\mmd^2[p,q]\]
and
\begin{flalign}
&\mE\big[\mmd_u^2[Y_k,\overline Y_k]-\mmd_u^2[Y_l,\overline Y_l]\big]<\alpha^2\mmd^2[p,q]-(1-\alpha)^2\mmd^2[p,q]+\xi,
\end{flalign}
for large enough $n$.

Therefore, we obtain,
\begin{flalign}
P\bigg(&\mmd_u^2[Y_k,\overline Y_k]-\mmd_u^2[Y_l,\overline Y_l]>0\bigg)\nn\\
=&P\bigg(\mmd_u^2[Y_k,\overline Y_k]-\mmd_u^2[Y_l,\overline Y_l]-\mE\big[\mmd_u^2[Y_k,\overline Y_k]-\mmd_u^2[Y_l,\overline Y_l]\big]\nn\\
&\quad\quad>-\mE\big[\mmd_u^2[Y_k,\overline Y_k]-\mmd_u^2[Y_l,\overline Y_l]\big]\bigg)\nn\\
\leq& P\bigg(\mmd_u^2[Y_k,\overline Y_k]-\mmd_u^2[Y_l,\overline Y_l]-\mE\big[\mmd_u^2[Y_k,\overline Y_k]-\mmd_u^2[Y_l,\overline Y_l]\big]\nn\\
&\quad>((1-\alpha)^2-\alpha^2)\mmd^2[p,q])-\xi\bigg),
\end{flalign}
for large enough $n$.

Applying McDiarmid's inequality, we obtain,
\begin{flalign}
P_e\leq& \exp\bigg(\log((n-s)s)-\frac{m((1-2\alpha)\mmd^2[p,q]-\xi)^2}{16K^2(1+ \Theta(\frac{1}{n}))}\bigg).
\end{flalign}
Since $\xi$ can be arbitrarily small, we conclude that if,
\begin{flalign}
m\geq\frac{16K^2(1+\eta)}{(1-2\alpha)^2\mmd^4[p,q]}\log(s(n-s)),
\end{flalign}
where $\eta$ is any positive constant, then $P_e\rightarrow 0$, as $n\rightarrow \infty$. It is also clear that if the above condition is satisfied, $P_e$ converges to zero exponentially fast with respect to $m$.

%
%

\section{Proof of Theorem \ref{thm:sunknownnonref}}\label{proof:sunknownnonref}
%
We analyze the performance of the test \eqref{test:sunknown_nonref}.
Without loss of generality, we assume that the first $s$ sequences are the anomalous sequences. Hence,
\begin{flalign}
P_e=P\bigg(&\big(\exists 1\leq l\leq s : \mmd_u^2[Y_l,\overline Y_l]\leq\delta_n \big)\text{or } \big( \exists s+1 \leq k\leq n:   \mmd_u^2[Y_k,\overline Y_k]>\delta_n   \big)\bigg)\nn\\
\leq\sum_{l=1}^{s}&P\bigg(\mmd_u^2[Y_l,\overline Y_l]\leq\delta_n\bigg)+\sum_{k=s+1}^nP\bigg(\mmd_u^2[Y_k,\overline Y_k]>\delta_n \bigg).
\end{flalign}
Using the fact that $\frac{s}{n}\rightarrow 0$ as $n\rightarrow \infty$,
and using \eqref{mmd:y1y1bar} and \eqref{mmd:ykykbar} we obtain,
\begin{flalign}
&\mE\big[\mmd_u^2[Y_l,\overline Y_l]\big]\rightarrow \mmd^2[p,q],\label{eq68}\\
&\mE\big[\mmd_u^2[Y_k,\overline Y_k]\big]\rightarrow 0,
\end{flalign}
as $n\rightarrow \infty$, for $1\leq l\leq s$ and $s+1\leq k\leq n$.

Due to \eqref{eq68}, for any constant $\epsilon$,   $-\mE\big[\mmd^2_u[Y_l,\overline Y_l]\big]<-\mmd^2[p,q]+\epsilon$ for large enough $n$.

For $1\leq l\leq s$, we drive,
\begin{flalign}
P&\bigg(\mmd_u^2[Y_l,\overline Y_l]\leq\delta_n\bigg)\nn\\
&=P\bigg(\mmd_u^2[Y_l,\overline Y_l]-\mE\big[\mmd_u^2[Y_l,\overline Y_l]\big]\leq -\mE\big[\mmd_u^2[Y_l,\overline Y_l]\big] +\delta_n\bigg)\nn\\
&\leq P\bigg(\mmd_u^2[Y_l,\overline Y_l]-\mE\big[\mmd_u^2[Y_l,\overline Y_l]\leq -(\mmd^2[p,q]-\epsilon-\delta_n)\bigg),
\end{flalign}
for large enough $n$.
%
Therefore, by applying McDiarmid's inequality, we obtain,
\begin{flalign}
P&\bigg(\mmd_u^2[Y_l,\overline Y_l]\leq\delta_n\bigg)\nn\\
&\leq \exp\bigg(   -\frac{2(\mmd^2[p,q]  -\epsilon-  \delta_n)^2}{\frac{16K^2}{m}(1+ \Theta(\frac{1}{n}))+\frac{16K^2}{m}(1+ \Theta(\frac{1}{n}))}  \bigg)\nn\\
&=\exp\bigg(   -\frac{m(\mmd^2[p,q]  -\epsilon-  \delta_n)^2}{16K^2(1+ \Theta(\frac{1}{n}))}  \bigg),
\end{flalign}
for large enough $n$.

For $s+1\leq k\leq n$,
\begin{flalign}
P&\bigg(\mmd_u^2[Y_k,\overline Y_k]>\delta_n \bigg)\nn\\
&=P\bigg(\mmd_u^2[Y_k,\overline Y_k]-\mE\big[\mmd_u^2[Y_k,\overline Y_k]\big]>\delta_n-\mE\big[\mmd_u^2[Y_k,\overline Y_k]\big]\bigg).
\end{flalign}
Using the fact that $\frac{s^2}{n^2\delta_n}\rightarrow 0$ as $n\rightarrow\infty$, we can show that
\[\frac{\mE\big[\mmd_u^2[Y_k,\overline Y_k]\big]}{\delta_n}\rightarrow 0,\]
as $n\rightarrow \infty$.
Hence, for large enough $n$, $\delta_n>\mE\big[\mmd_u^2[Y_k,\overline Y_k]\big]$. Therefore, using McDiarmid's inequality, we have
\begin{flalign}\label{eq72}
P&\bigg(\mmd_u^2[Y_k,\overline Y_k]>\delta_n \bigg)\nn\\
&\leq \exp\bigg(   -\frac{2(\delta_n-\mE\big[\mmd_u^2[Y_k,\overline Y_k]\big] )^2}{\frac{16K^2}{m}(1+ \Theta(\frac{1}{n}))+\frac{16K^2}{m}(1+ \Theta(\frac{1}{n}))}  \bigg)\nn\\
&=\exp\bigg(   -\frac{m(\delta_n-\mE\big[\mmd_u^2[Y_k,\overline Y_k]\big])^2}{16K^2(1+ \Theta(\frac{1}{n}))}  \bigg).
\end{flalign}
Therefore,
\begin{flalign}
P_e&\leq s\exp\bigg(   -\frac{m(\mmd^2[p,q]  -\epsilon-  \delta_n)^2}{16K^2(1+ \Theta(\frac{1}{n}))}  \bigg)\nn\\
&+(n-s)\exp\bigg(   -\frac{m(\delta_n-\mE\big[\mmd_u^2[Y_k,\overline Y_k]\big])^2}{16K^2(1+ \Theta(\frac{1}{n}))}  \bigg)\nn\\
&=\exp\bigg( \log s  -\frac{m(\mmd^2[p,q]  -\epsilon-  \delta_n)^2}{16K^2(1+ \Theta(\frac{1}{n}))}  \bigg)\nn\\
&+\exp\bigg(\log(n-s)   -\frac{m(\delta_n-\mE\big[\mmd_u^2[Y_k,\overline Y_k]\big])^2}{16K^2(1+ \Theta(\frac{1}{n}))}  \bigg),
\end{flalign}
for large enough $n$.
Hence, we conclude that if
\begin{flalign}
m\geq\frac{16(1+\eta)K^2}{(\mmd^2[p,q] -  \delta_n)^2}\log s,
\end{flalign}
and
\begin{flalign}
m\geq\frac{16(1+\eta)K^2}{(\delta_n-\mE\big[\mmd_u^2[Y_k,\overline Y_k]\big])^2}\log(n-s),
\end{flalign}
where $\eta$ is any positive constant, then $P_e\rightarrow 0$, as $n\rightarrow \infty$.

When $s=0$, $P_e=\sum_{k=1}^nP(\mmd_u^2[Y_k,\overline{Y_k}]>\delta_n)$. Then applying \eqref{eq72}, we have if
\begin{flalign}
m\geq\frac{16(1+\eta)K^2}{(\delta_n-\mE\big[\mmd_u^2[Y_k,\overline Y_k]\big])^2}\log n,
\end{flalign}
where $\eta$ is any positive constant, then $P_e\rightarrow 0$, as $n\rightarrow \infty$.


\section{Proof of Proposition \ref{thm:conv}}\label{sec:proofcon}
We first introduce an interesting property of Gaussian distribution, which is useful for bounding the probability of error for our problem.

\begin{lemma}\cite{Hall1979}\label{lemma:extreme}
  For the standard Gaussian distribution with mean zero and variance one, there exists positive constants $c_1$ and $c_2$  such that the cumulative distribution function (CDF) $\Phi(x)$ of the standard Gaussian distribution satisfies the following inequalities:
  \begin{flalign}
    \frac{c_1}{\log n}< \underset{-\infty<x<\infty}{\sup}|\Phi^n(a_nx+b_n)-G(x)|<\frac{c_2}{\log n}
  \end{flalign}
  for all positive integer $n$, where $G(x)=e^{e^{-x}}$ (i.e., the CDF of the Gumbel distribution), $a_nb_n=1$. In particular, $b_n$ can be approximated as
\begin{flalign}
  b_n=\sqrt{2\log n}-\frac{\frac{1}{2}\log (4\pi \log n)}{\sqrt{2\log n}}+O\left(\frac{1}{\log n}\right).
\end{flalign}
\end{lemma}

Our main idea of the proof is to show that under a certain distribution pair $p$ and $q$, even the optimal parametric test is not consistent under the condition given in the theorem. This thus implies that under the same condition, no nonparametric test is universally consistent for arbitrary $p$ and $q$. Towards this end, we consider the case, in which $p$ and $q$ are Gaussian with the same variance but mean shift, i.e., $p=\mathcal{N}(0,1)$ and $q=\mathcal N(1,1)$. The optimal test with known $p$ and $q$ is the following maximum likelihood (ML) test.
\begin{flalign}
  \hat i=\underset{1\leq i\leq n}{\arg \max} \{P_i(Y^{nm})\},
\end{flalign}
where $P_i(Y^{nm})$ denotes the probability of $Y^{nm}$ if the $i$-th sequence is anomalous.
The probability of error  under the ML test is given by:
\begin{flalign}
  P_e=\frac{1}{n}\sum_{i=1}^n \mathcal P_i\Big(P_i(Y^{nm})\leq \underset{k\neq i}{\max} P_k(Y^{nm})\Big),
\end{flalign}
where $\mathcal P_i$ denotes the probability evaluated when $i$-th sequence is anomalous.
By the symmetry of the problem,
\begin{flalign}
\mathcal P_i\Big(P_i(Y^{nm})\leq \underset{k\neq i}{\max} P_k(Y^{nm})\Big)=\mathcal P_j\Big(P_j(Y^{nm})\leq \underset{k\neq j}{\max} P_k(Y^{nm})\Big),
\end{flalign}
for any $1\leq i,j\leq n$.
Hence, we have
\begin{flalign}
  P_e&= \mathcal P_1\Big(P_1(Y^{nm})\leq \underset{k\neq 1}{\max} P_k(Y^{nm})\Big)\nn\\
  &=\mathcal P_1\Big(\frac{1}{\sqrt{m}}\sum_{ i=1}^m Y_{1i}\leq \underset{2\leq k\leq n}{\max}\frac{1}{\sqrt{m}}\sum_{i=1}^mY_{ki}\Big).
\end{flalign}
For convenience, we  define   $B_1:=\frac{1}{\sqrt{m}}\sum_{ i=1}^m Y_{1i}$, and  $B_k:=\frac{1}{\sqrt{m}}\sum_{i=1}^mY_{ki}$, for $2\leq k\leq n$. Hence, $B_1\sim \mathcal N(\sqrt{m},1)$, and $B_k\sim \mathcal N(0,1)$, and they are independent from each other. With the above definitions, the probability of error can be written as
\begin{flalign}
  P_e&=\mathcal P\Big(B_1\leq \underset{2\leq k\leq n}{\max} B_k\Big)\nn\\
  &=1-\mathcal P\Big(\underset{2\leq k\leq n}{\max} B_k <B_1\Big)\nn\\
  &=1-\mathbb{E}_B \Big\{\Phi^{n-1}(B_1)\Big\}
\end{flalign}
where $\Phi$ is the CDF of $B_k$.

By Lemma \ref{lemma:extreme}, there exists a constant $c$ independent of $n$, such that for all positive integer $n$, and for all real values $x$,
\begin{flalign}
  G\Big(\frac{x-b_n}{a_n}\Big)-\frac{c}{\log n}\leq \Phi^n(x)\leq G\Big(\frac{x-b_n}{a_n}\Big)+\frac{c}{\log n},
\end{flalign}
where $a_n$, $b_n$ are optimal normalizing constants, and $G(x)=e^{-e^{-x}}$ is the CDF of the Gumbel distribution.

Hence,
\begin{flalign}\label{eq:64}
  P_e&=1-\mathbb{E}_B \Phi^{n-1}(B_1)\nn\\
  &\geq 1-\frac{c}{\log (n-1)}-\mathbb{E}_B\Big\{G\Big(\frac{B_1-b_{n-1}}{a_{n-1}}\Big)\Big\}\nn\\
  &=1-\frac{c}{\log (n-1)}-\mathbb{E}_T\Big\{G(T)\Big\},
\end{flalign}
where $T=\frac{B_1-b_{n-1}}{a_{n-1}}$, and $T\sim \mathcal N(\frac{\sqrt{m}-b_{n-1}}{a_{n-1}},\frac{1}{a_{n-1}^2})$. The second term in \eqref{eq:64} can be further bounded as
\begin{flalign}
  \mathbb{E}_T&\Big\{G(T)\Big\}\nn\\
  &=\int_{-\infty}^0 e^{-e^{-t}}p(t)dt+\int_0^{+\infty} e^{-e^{-t}}p(t)dt\nn\\
  &\leq e^{-1}+P(T\geq 0)
\end{flalign}
where
\begin{flalign}
  P(T\geq 0)&=Q\left(\frac{0-\frac{\sqrt m-b_{n-1}}{a_{n-1}}}{\frac{1}{a_{n-1}}}\right)=Q(b_{n-1}-\sqrt m).
\end{flalign}
In the above equations, $Q(\cdot)$ denotes the tail probability of the standard Gaussian distribution.
If $m\leq{2(1-\eta)\log n}$, where $\eta$ is any positive constant, $b_{n-1}-\sqrt m\rightarrow \infty$, $Q(b_{n-1}-\sqrt m)\rightarrow 0$. Hence,
\begin{flalign}
  \lim_{n\rightarrow \infty}\mathbb E_T[{G(T)}]\leq e^{-1}.
\end{flalign}
Thus, with $\frac{c}{\log n}\rightarrow 0$
\begin{flalign}
\lim_{n\rightarrow\infty}P_e\geq 1-e^{-1}\approx 0.6321 >0
\end{flalign}
as $n \rightarrow \infty$. Therefore, if $m=O (\log n)$, where $\eta$ is any positive constant, there exists no consistent test for any arbitrary distributions $p$ and $q$.
\renewcommand{\baselinestretch}{1}

\bibliographystyle{unsrt}
\bibliography{SequenceDetection}


\end{document}